\documentclass{article}

\usepackage[preprint]{neurips_2019}

\usepackage[utf8]{inputenc} 
\usepackage[T1]{fontenc}    
\usepackage{hyperref}       
\usepackage{url}            
\usepackage{booktabs}       
\usepackage{amsfonts}       
\usepackage{amsmath}
\usepackage{amsthm}
\usepackage{nicefrac}       
\usepackage{microtype}      
\usepackage{graphicx}
\usepackage{soul}

\graphicspath{{./figures/}}
\usepackage{color}          

\newcommand{\RR}{\mathbb{R}}
\newcommand{\defeq}{\stackrel{\textrm{def}}{=}}
\newcommand{\NNF}[3]{\textrm{NNF}_{#1}^{#2\rightarrow #3}}
\newcommand{\OO}{\mathcal{O}}
\newcommand{\vin}{\rotatebox[origin=c]{-90}{$\in$}}

\newcommand\restr[2]{{
		\left.\kern-\nulldelimiterspace 
		#1 
		\vphantom{\big|} 
		\right|_{#2} 
}}
\newcommand{\IFS}[2]{\mathcal{#1} = \{#2\}}

\newtheorem{Theorem}{Theorem}
\newtheorem{Lemma}{Lemma}
\newtheorem{Definition}{Definition}

\renewcommand{\P}{\mathcal{P}}
\newcommand{\one}{\mathbf{1}}
\newcommand{\relu}{\textrm{ReLU}}
\newcommand{\K}{\mathcal{K}}
\title{ Expression of Fractals Through Neural Network Functions}             

\author{Nadav Dym*\\
	Department of Mathematics and Rhodes Information Initiative\\
	Duke University\\
	Durham, NC 27708 \\
	\texttt{nadavdym@gmail.com} \\
	\And
	Barak Sober* \\
	Department of Mathematics and Rhodes Information Initiative\\
	Duke University\\
	Durham, NC 27708 \\
	\texttt{barakino@math.duke.edu} \\
	\And
	Ingrid Daubechies \\
	Department of Mathematics and Rhodes Information Initiative\\
	Department of Electrical and Computer Engineering\\
	Duke University\\
	Durham, NC 27708 \\
	\texttt{ingrid@math.duke.edu} \\
}

\begin{document}
	\maketitle
	\begin{abstract}
		To help understand the underlying mechanisms of neural networks (NNs), several groups have, in recent years, studied the number of linear regions
		$\ell$ of piecewise linear functions generated by deep neural networks (DNN).
		In particular, they showed that $\ell$ can grow exponentially with the number of network parameters $p$, a property often used to explain the advantages of DNNs over shallow NNs in approximating complicated functions.
		Nonetheless, a simple dimension argument shows that DNNs cannot generate all piecewise linear functions with $\ell$ linear regions as soon as $\ell > p$.
		It is thus natural to seek to characterize specific families of functions with $\ell$ linear regions that can be constructed by DNNs.
		Iterated Function Systems (IFS) generate sequences of piecewise linear functions $F_k$ with a number of linear regions exponential in $k$.
		We show that, under mild assumptions, $F_k$ can be generated by a NN using only $\mathcal{O}(k)$ parameters.
		IFS are used extensively to generate, at low computational cost, natural-looking landscape textures in artificial images.
		They have also been proposed for compression of natural images, albeit with less commercial success. The surprisingly good performance of this fractal-based compression suggests that our visual system may lock in, to some extent, on self-similarities in images.
		The combination of this phenomenon with the capacity, demonstrated here,
		of DNNs to efficiently approximate IFS may contribute
		to the success of DNNs, particularly striking for image processing tasks,  as well as suggest new algorithms for representing self similarities in images based on the DNN mechanism.

	\end{abstract}

	\bigskip\nobreak
	In the past few years Deep Neural Networks (DNN) have been remarkably successful in solving problems in computer vision (e.g., \cite{krizhevsky2012imagenet,cirecsan2012multi,qi2017pointnet}), natural language processing (e.g., \cite{manning2014stanford,collobert2008unified}), signal processing (e.g., \cite{Piczak2015Environmental}), and even art investigation (e.g., \cite{Sabetsarvestani2019Artificial}).
	This success has spawned a vast body of literature focusing on both practical and theoretical aspects of DNNs.

	One of the main topics in the theoretical investigation of DNNs is the understanding of their expressiveness and approximation power. The universality theorem shows that even NNs with one hidden layer can approximate any continuous function  \cite{cybenko1989approximation,Hornik1989Multilayer}.
	However, these results do not give a reasonable explanation to the power of depth, which had shown to be pivotal to the success of DNNs.
	In recent years, several works showed that shallow networks are not as efficient as deep ones in terms of their approximation power \cite{poggio2017and,yarotsky2017error,liang2016deep,cohen2016expressive,eldan2016power}.

	In contrast with the study of approximation, the study of  expressiveness focuses on understanding the class of  functions which DNN can express \emph{exactly}.
	As modern DNNs are typically composed of affine transformations and linear activation units (ReLU), functions expressed by DNNs must be continuous and piecewise linear (CPwL) by construction.
	Moreover, it is  well established that neural networks with $p$ parameters can express standard families of CPwL functions with $\OO(p)$ linear regions such as free knot splines and CPwL bases over triangulations  (\cite{arora2018understanding,yarotsky2017error,triangles,ingrid}).
	As these results do not necessarily require the network to be deep this can be considered as an analogue of the universality theorem in terms of expressiveness.
	However, DNNs can also express CPwL functions with $\ell \gg p$ linear regions.
	Several recent works  \cite{arora2018understanding,serra2017bounding,montufar2017notes, montufar2014number}) have studied the  maximal number of linear regions generated by DNNs.
	Their results show that for DNNs $\ell$ can in fact grow exponentially faster than $p$, if the depth of the network is allowed to grow.
	This property is often used as a possible explanation of the advantages of DNNs over shallow networks in expressing complex functions.

	In this paper our aim is to further our understanding of the class of CPwL functions in the domain $\ell \gg p $.
	Using a simple dimension argument we show that DNNs with $p$ parameters cannot span all piecewise linear functions
	with  $\ell \gg p$ linear regions (see Lemma~\ref{lem:dim} in Appendix~\ref{sec:additional}).
	What then are the CPwL functions which can be expressed by a DNN with a small number of parameters? To the best
	of our knowledge there is not much known on this topic, a notable exception being the works of
	\cite{ingrid, DBLP:journals/corr/abs-1901-02220}  which construct some examples of such scalar functions.

	Our main result is that DNNs are capable of expressing fractal-like functions, such as the (almost) indicator
	function of the fractals shown in figure~\ref{fig:IFS}. This construction is applicable for a wide
	family of fractals arising from \emph{Iterated Functions Systems} (IFS) satisfying some weak conditions.

	IFS generate intricate fractal-type shapes by iteratively applying all possible combinations of a small
	fixed number of functions.
	Throughout the history of computer vision research IFS have played an important role \cite{welstead1999fractal}.
	IFS make it possible to generate, at low computational cost, natural-looking landscape textures \cite{barnsley2014fractals}; they
    can be used for  artificial generation of such images in entertainment applications such as movies
	and video games \cite{van1996dynamic}. Self similarity in images has been suggested as a tool for several tasks in computer vision including
	deblurring, dehazing and super resolution \cite{bahat2016blind,michaeli2013nonparametric,irani2009super}.  IFS can also obtain surprisingly good results in image compression \cite{jacquin1992image},
	by encoding geometric similarity of small parts of an (arbitrary) image with other, slightly larger parts elsewhere in
	the same (natural) image. That images compressed and then reconstructed via this coding are perceived as reasonable approximations of
	the original image may be connected as much (or more) with our visual system as with the properties
	of natural images; this observation is reminiscent of the role played by similar comparisons
	between parts of images in \cite{smale2010mathematics}.
	The particular efficiency of NNs at encoding functions
	or shapes generated by IFS, demonstrated in this paper, is intriguing
	in light of this possible
	connection of IFS with the perception of natural images through our own biological
	neural system, together with the observation that NNs were
	originally inspired, at least in part, by an (idealized) version of
	the network formed by biological neurons.

	\section{Problem setup and main results}\label{sec:problemSetup}
	Prior to quoting our main result, we wish to introduce some notations pertaining to neural networks, and give a brief introduction to \emph{iterated function systems}.

	\subsection{Neural network notation}
	The term neural network is being used in the literature to describe both learning algorithms (based upon artificial neural network architecture) and realizations of such architectures; i.e., a given architecture with fixed coefficients to be used as a function.
	In this paper, we focus on a family of functions that can be expressed through such architecture.
	Thus, to emphasize this distinction, throughout the paper we refer to such functions as Neural Network Functions (NNF).
	Furthermore, we limit our discussion to fully connected architectures with ReLU  non-linear activation. We also restrict our attention to the behavior of $f$ on some fixed arbitrarily large compact set $\K$.
	\begin{Definition}[ReLU]
		Let $x = (x_1, \ldots, x_N)^T$, then
		\begin{equation}\label{eq:ReLU}
		\emph{ReLU}(x) \defeq (\max(x_1,0), \ldots, \max(x_N,0))^T
		,\end{equation}
	\end{Definition}

	\begin{Definition}[NNF]
		A \emph{neural network function} is a function of the form
		\begin{equation} \label{eq:NNF}
		f(x) = \eta_{1} \circ \eta_{2} \circ \cdots \circ \eta_{2d+1}(x), \forall x \in \K
		,\end{equation}
		where $\eta^{}_{2i-1} $ is an affine transformation and  $\eta^{}_{2i}=\emph{ReLU}$ for all $i=1,\ldots, d$.
	\end{Definition}
	The representation of $f$ in the form \eqref{eq:NNF} is non-unique, and we refer to a specific representation as $f_{\eta} $. Given such a representation we denote $L(f_\eta)=d $ and refer to it as the \emph{depth} of $f_\eta $ . The \emph{width} of $ f_\eta $ is the maximal input or output dimension of all $ \eta^{}_i $ comprising $ f_\eta $ and is denoted by $ W(f_\eta) $.
	We denote the family of NNFs with depth at most $ d $ and width at most $ w $ by
	\begin{equation}\label{eq:NNF_def}
	\NNF{w,d}{n_1}{n_2} \defeq \left\lbrace f:\K \subseteq \RR^{n_1}\rightarrow\RR^{n_2} |\, \exists f_\eta \text{ s.t. } f=f_\eta,  L(f_\eta)\leq d,~W(f_\eta)\leq w \right \rbrace
	,\end{equation}
	Note that the number of parameters needed to express $ f\in\NNF{w,d}{n_1}{n_2} $ is $\OO(w^2 \cdot d)$.

	\begin{figure}[t]
		\includegraphics[width=\columnwidth]{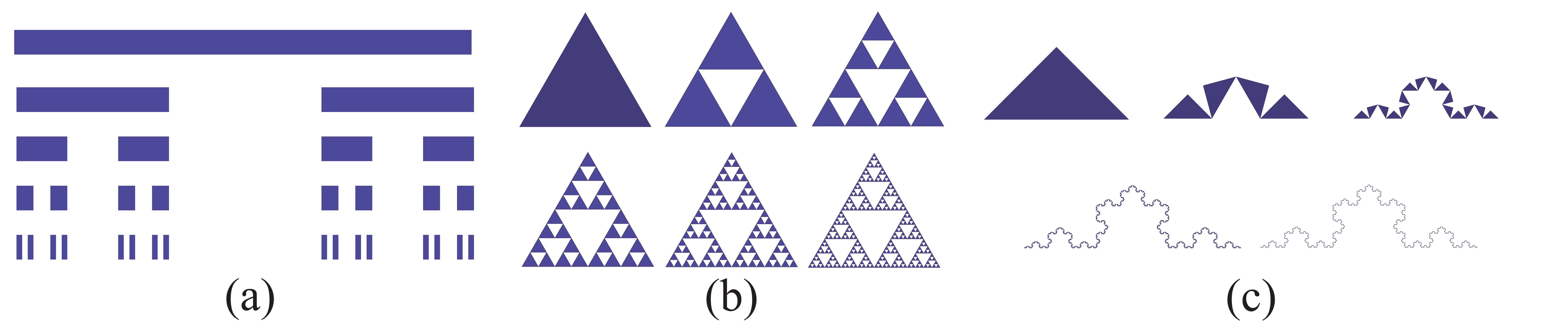}
		\caption{The first few sets in the iterative construction of (a) the Cantor set, (b) The Sierpinski triangle, and (c) the Koch curve.}
		\label{fig:IFS}
	\end{figure}
	\subsection{Iterated Function System}
	We provide a short review on \emph{Iterated Function System} (IFS). The results stated here can be found in the classical book on  IFS \cite{barnsley2014fractals}.

	An  IFS on $\RR^d$ is a method for constructing very detailed patterns from a small number of building blocks.
	Mathematically, they are defined as a finite collection $ \IFS{F}{f_1 \ldots, f_J} $ of contractive mappings $f_j:\RR^d\to \RR^d$.
	Given an initial bounded subset $A_0 \subset \RR^d $, an IFS generates a sequence of increasingly detailed sets $A_k $ by recursively applying the \emph{Hutchinson operator}
	$$ H(A)=\bigcup\limits_{j=1}^J f_j A,$$
	to obtain a sequence of sets
	\begin{equation} \label{eq:Ak}
	A_{k+1}=H(A_k). \end{equation}
	For any compact set $A_0$, the sequence $A_k=H^kA_0 $ converges in the Hausdorff metric to a compact set $K$ which is called the attractor of the IFS.
	However, despite  the asymptotic independence on the choice of the initial set, in practice from visual as well as computational considerations it is useful to work with special sets which we will call
	``nice'' sets:
	\begin{Definition}[Nice sets]
		We say that a bounded set $A$ is nice with respect to $\IFS{F}{f_1,\ldots,f_J} $, if
		\begin{equation} \label{eq:contained}
		f_j A \subseteq A, j=1,\ldots,J
		\end{equation}
		and
		\begin{equation}\label{eq:disjoint}
		f_i A \cap f_j A=\emptyset, \, \forall i \neq j,
		\end{equation}
	\end{Definition}
	The existence of nice sets is an imporant property in the study of IFS (\cite{barnsley2014fractals} page 129):
	\begin{Definition}[Totally disconnected]\label{def:totally}
		An \emph{IFS} is \emph{totally disconnected}  if it possesses a nice compact set.
	\end{Definition}
	\begin{Definition}[Just touching]\label{def:just}
		An \emph{IFS} is \emph{just touching} if it possesses a nice open set.
	\end{Definition}

	From here on, for the sake of clarity and convenience in notation, we will denote compact sets by the character $ C $ or $K$, open sets by $ U $ and general sets by $ A $.

	\subsubsection{Examples}
	\textbf{The cantor set:}
	The cantor set $K_{\textrm{Cantor}}$ is the attractor of an IFS on $\RR$ defined by $ \IFS{F}{f_1, f_2} $, where
	\begin{equation}\label{eq:cantorFun}
	f_1(x)=1/3x \text{ and } f_2(x)=1/3x+2/3 .
	\end{equation}
	The IFS $\mathcal{F}$ is totally disconnected since the set $C_0=[0,1]$ is a nice compact set. Figure~\ref{fig:IFS}(a) shows the first few sets  $C_i$ obtained from this choice of $C_0$.

	\textbf{Just touching} Figure~\ref{fig:IFS}(b)-(c) shows two examples of \emph{just touching} IFS which generate the Sierpinski triangle and the Koch curve. A full description of these IFS is given in the Appendix for completeness.

	\subsection{Main results}
	Our goal is to use NNF to efficiently construct functions which represent sets $A_k$ of the form \eqref{eq:Ak}. The classical choice of a function which represents a set $A$ is  the standard indicator function
	\begin{equation}
	\chi^{}_{A}(x)  = \left\lbrace\begin{array}{ll}
	1 ~,& x\in {A} \\
	0 ~,& x\notin {A}
	\end{array}\right.
	.\end{equation}
	However, this function cannot be realized as an NNF as it is not continuous.
	Instead, we define the notion of a  CPwL \emph{indicator function}
	\begin{Definition} We say that $ \varphi^{}_{C}(x) $ is a \emph{CPwL indicator function} of the compact set $C$, if
		$$
		\left\lbrace\begin{array}{ll}
		\varphi^{}_{C}(x)\geq 0 ~,& x\in {C} \\
		\varphi^{}_{C}(x)< 0 ~,& x\notin {C}
		\end{array}\right.
		.$$
	\end{Definition}
	Note that a compact set $C$ may admit many CPwL indicator functions $\varphi^{}_C$.
	CPwL indicators can be used to approximate indicator functions to arbitrary precision due to the following simple lemma.
	\begin{Lemma}\label{lem:indicator}
		If $\varphi^{}_C \in \NNF{W,L}{d}{1}	$ is a \emph{CPwL} indicator function of $C$, then there exist $\psi_n \in \NNF{\max\{W,2\},L+1}{d}{1} $ which converge to $ \chi^{}_{C}$ pointwise and in $L^p(\K) $ for all  $1 \leq p < \infty $.
	\end{Lemma}
	The proof of this lemma is given in Appendix \ref{sec:additional}.

	We now state our main result:
	\begin{Theorem}\label{thm:strong}
		Let $\IFS{F}{f_1,\ldots,f_J} $ be an IFS on $\RR^d$. Let $C_0$ be a nice  compact subset of $\RR^d$, and assume
		\begin{enumerate}
			\item $ C_0 $ is a finite union of convex polytopes.
			\item $f_j$ are invertible affine transformations.
		\end{enumerate}
		Then, there exist constants $W_0,L_0 $, such that for all $k$, there exists a \emph{CPwL} indicator function $\varphi_k \in \NNF{W_0,L_0k}{d}{1}$ \, for the set $C_k=H^kC_0 $.
	\end{Theorem}
	The constants $W_0,L_0 $ are given in Equation~\eqref{eq:constants}.
	They depend on $d,J$, the number $p_0$ of convex polytopes which compose $C_0$, and the maximal number of hyperplanes $m$ needed to express each of the polytopes.


	One example for which Theorem~\ref{thm:strong} applies is the cantor set we discussed earlier, with the initial compact nice set $C_0=[0,1] $.
	The set $C_0$ is an interval- a 1D polytope, and the functions $f_1,f_2$ defined in \eqref{eq:cantorFun} are affine and invertible.
	In this case  $d=1,p_0=1,J=2,m=2 $, and so the theorem together with the parameter count in \eqref{eq:constants} shows that we can represent a CPwL indicator function for $C_k$ as an NNF with width $10$ and depth $4k$.

	We note that by combining the theorem with Lemma~\ref{lem:indicator} we also obtain that the indicator $\chi_{C_k}$ can be approximated to machine precision using a NNF with almost identical complexity. Finally we note that as can be seen in \eqref{eq:rnn} the DNN construction we propose is composed of an identical network applied $k$ times to obtain a CPwL indicator of $C_k$, and so by using a recurrent NN architecture instead of the fully connected architecture we described here the number of parameters necessary to  represent $C_k$ can be completely independent of $k$.

	We now turn to discuss \emph{just touching} IFS; that is, IFS which possesses a nice open set $U$. In this case, we can use the mechanism of the proof of Theorem \ref{thm:strong} to approximate $\chi_{U_k}$ to arbitrary precision, though we cannot construct an exact CPwL indicator in this case:
	\begin{Theorem}\label{thm:open}
		Let $\IFS{F}{f_1,\ldots,f_J} $ be an IFS on $\RR^d$. Let $U_0$ be a nice open subset of $\RR^d$, and assume
		\begin{enumerate}
			\item $ U_0 $ is a finite union of disjoint open convex polytopes.
			\item $f_j$ are invertible affine transformations.
		\end{enumerate}
		Then, there exist constants $W_0,L_0 $, such that for all $k$, there exist a sequence of functions $\psi_n \in \NNF{W_0,L_0k+1}{d}{1}$ \, which converge to the indicator function of $U_k=H^kU_0 $ pointwise and in $L^p(\K) $ for all compact $\K $ and $1 \leq p < \infty $.
	\end{Theorem}

	Finally, we note that any IFS on $\RR^d$ can be lifted to a totally disconnected IFS on $\RR^{d+1}$. We discuss this simple construction in Appendix~\ref{app:lift}.

	The remainder of the paper is organized as follows: In Section~\ref{sec:prelim} we discuss some theoretical preliminaries, which will be useful for the proof of Theorem~\ref{thm:strong}, presented in Section~\ref{sec:proof}. An outline of the proof of Theorem~\ref{thm:open}, the proof of  Lemma~\ref{lem:indicator}, and the dimension argument mentioned in the introduction (i.e., the fact that DNNs with $ p $ parameters cannot span all CPwL with $ \ell \gg p $ linear regions), can be found in  Appendix~\ref{sec:additional}.
	\section{Theoretical preliminaries}\label{sec:prelim}
	In this section we wish to define formally some basic actions on NNF as well as two NNF \emph{building blocks} that will be used in our construction later on.
	The actions we wish to address are \emph{composition}, \emph{splitting}, and \emph{summation}; and the building blocks are the commonly used \emph{min} and \emph{max} NNFs.

	\subsection{Basic actions}\label{sub:basic}
	\textbf{Composition:} Using the aforementioned notation, the composition of two neural network functions is
	\begin{equation}\label{eq:NNFcomposition}
	f\in\NNF{w_1,d_1}{n_1}{n_2} ~,~ g\in\NNF{w_2,d_2}{n_2}{n_3} ~\Rightarrow~
	g\circ f \in \NNF{\max(w_1,w_2),d_1+d_2}{n_1}{n_3}
	.\end{equation}
	Accordingly, the number of parameters in $ g\circ f $ is $ \OO\left(\max(w_1,w_2)^2\cdot (d_1+d_2)\right) $.

	\textbf{Splitting:} It is sometimes desirable to split a function into two separate operations
	\begin{equation}\label{eq:split}
	\begin{array}{ccc}
	& & f(x) \\
	&\nearrow & \\
	x & & \\
	\vin & \searrow & \\
	\RR^n & & g(x)
	\end{array}
	\end{equation}
	where $ f\in\NNF{w_1,d_1}{n}{n_1} ~,~ g\in\NNF{w_2,d_2}{n}{n_2} $.
	The NNF resulting from the action of splitting \eqref{eq:split} is denoted by $ x^f_g $ and
	\begin{equation}\label{eq:splitNNF}
	x^f_g \in \NNF{w_1 + w_2, \max(d_1,d_2)}{n}{(n_1 + n_2)}
	.\end{equation}
	In the case where $d_1=d_2 $ the definition of $x^f_g $ is straightforward. The case $d_1<d_2 $ can be reduced to the case $d_1=d_2$ by representing $f$ using a network of depth $d_2$. This is done by shifing $f$ by such a large value $b $ so that $f(\K)+b \geq 0 $, and so $f+b $, will be unaffected by ReLU, then applying identity mappings, and finally reapplying translation by $-b$ to achieve the original function $ f$.

	%
	\textbf{Summation:} Finally, given $ f\in \NNF{w_1, d_1}{n_1}{n_2} $ and $ g\in \NNF{w_2, d_2}{n_1}{n_2} $, the function $ f+g $ is well defined and
	\begin{equation}
	f + g =
	\begin{array}{ccccc}
	& & f(x) &&\\
	&\nearrow &&\searrow& \\
	x & & &&f(x)+g(x)\\
	\vin & \searrow &&\nearrow& \\
	\RR^{n_1} & & g(x)&&
	\end{array}
	.\end{equation}
	Thus, $ f+g $ is merely $ x^f_g $ with an added linear layer as output, and so
	\begin{equation}
	f+g\in \NNF{(w_1+ w_2), max(d_1, d_2)+1}{n_1}{n_2}
	.\end{equation}

	\subsection{min and max as NNF}\label{sub:minmax}
	Although common neural network architectures apply min and max operations (e.g., max-pooling), in order to maintain consistency with the standard definition of fully connected networks we define them here as NNF using ReLU.
	We can pronounce the max and min operation as
	\begin{equation}\label{eq:maxAndMin}
	\max\{x,y\} = \textrm{ReLU}(x-y) + y \text{,  and  } \min\{x,y\} = x- \textrm{ReLU}(x-y)
	,\end{equation}
	Thus, for $ x, y\in \RR $
	\begin{equation}
	\min\{x,y\}, \max\{x,y\}\in \NNF{2,1}{2}{1}
	.\end{equation}
	For a vector $ x\in\RR^\ell $ we define
	\begin{equation}
	\min(x) \defeq \min\{x_1,\ldots,x_\ell\}
	.\end{equation}
	This minimum operation  can be implemented via a \emph{Divide and Conquer} type approach (e.g., see \cite{triangles}, Theorem 3.1) with $\log_2 \ell$ depth, and therefore
	\begin{equation}\label{eq:vecMin}
	\min(x)\in \NNF{\ell, \lceil\log_2(\ell) \rceil}{\ell}{1}
	.\end{equation}

	Furthermore, for $ x\in \RR^\ell $ and $y \in \RR $ we denote the vector-scalar minimum  operation by
	\begin{equation}
	m\{x,y\} \defeq (\min\{x_1, y\},\ldots,\min\{x_\ell, y\})
	,\end{equation}
	and it follows that
	\begin{equation}\label{eq:minmaxNNF}
	m\{x,y\},\in \NNF{\ell+1,1}{\ell+1}{\ell}
	.\end{equation}
	\section{Proof of Theorem~\ref{thm:strong}}\label{sec:proof}


	The outline of the proof is as follows:
	\begin{enumerate}
		\item\label{outline:T} Since $ C_0 $ is a nice set with respect to the IFS, there exists a function $ T:\RR^d\rightarrow\RR^d $, which satisfies $ T\circ f_j(x) = x $.
		\item\label{outline:TCkEquivalence} Then, we can establish the following relation (see Lemma \ref{lem:CkiffTk})
		\begin{equation}
		x\in C_k \iff  x,Tx,\ldots,T^{k-1}x \in C_1
		.\end{equation}
		\item\label{outline:TandPhiNNF} Both $ T $ and $ \varphi^{}_{1} $ (a CPwL indicator of $ C_1 $) can be constructed as NNF (see Sections \ref{sub:phi} and \ref{sub:T}).
		\item\label{outline:phik} Then, from the construction of vector \emph{min} as NNF and the fact that composition of NNFs is NNF (see Sections \ref{sub:basic} and \ref{sub:minmax}), we get that a CPwL indicator of $ C_k $ can be pronounced as the following NNF
		\begin{equation}\label{eq:phik}
		\varphi_k(x)=\min(\varphi_1(x),\varphi_1(Tx),\ldots \varphi_1(T^{k-1}x))
		.\end{equation}
		\item\label{outline:TphikCoeffs} Using the basic actions' calculations for width and depth (Section \ref{sub:basic}) we come to the conclusion that
		\begin{equation}
		\varphi_k \in \NNF{W_0, L_0 k}{d}{1}
		,\end{equation}
		where $ W_0, L_0 $ are some constants independent of $ k $ (see Section \ref{sub:phikConstruction}).
	\end{enumerate}

	Let us now show that if there exists a function $ T $, as described in item \ref{outline:T} in the outline, then the equivalence relation of item \ref{outline:TCkEquivalence} holds.
	\begin{Lemma}\label{lem:CkiffTk}
		Let $\{f_1,\ldots,f_J \}$ be an IFS, and assume $C_0 $ is a nice set, and $f_j C_0 \subseteq C_0, \forall j $. Assume $T:\RR^d \to \RR^d $ satisfies $T \circ f_j(x)=x $ for all $x \in C_0 $,
		then
		$$ x\in C_k \iff  x,Tx,\ldots,T^{k-1}x \in C_1.$$
	\end{Lemma}
	\begin{proof}
		Assume $x\in C_k$, then there exist $j_1,j_2,\ldots j_k \in \{1,\ldots,J\} $, and $y \in C_0$, such that
		$$x=f_{j_1}\circ f_{j_2}\ldots \circ f_{j_k}(y) $$
		Thus for all $0 \leq r \leq k-1 $
		$$T^rx=f_{j_{r+1}} \circ f_{j_{r+2}}\ldots \circ f_{j_k} (y) \in C_{k-r} \subseteq C_1$$
		as required.

		We prove the opposite direction by induction on $k$. For $k=1$ the claim is obvious. Now assume the claim holds for $k-1$, we want to show that it holds for $k$ as well- we assume $x,Tx,\ldots T^{k-1}x \in C_1 $ and we need to show that $x \in C_k$.

		Since $x\in C_1 $, there exists $y_1 \in C_0 $ and $j_1 \in \{1,\ldots,J \}$, such that
		\begin{equation}\label{eq:x}
		x=f_{j_1}(y_1).
		\end{equation}

		By applying $T$ to this equation we obtain
		\begin{equation}\label{eq:Tx}
		Tx=y_1 .
		\end{equation}
		On the other hand by the induction hypothesis we have that $Tx \in C_{k-1} $ so that there exists $y_2 \in C_0 $ and $j_2,\ldots,j_k \in \{1,\ldots,J \}$ such that
		\begin{equation}\label{eq:Tx2} Tx=f_{j_2}\circ f_{j_3} \ldots \circ f_{j_k} (y_2)  \end{equation}
		and so it follows that $x \in C_k $ since
		$$ x\stackrel{\eqref{eq:x}}{=} f_{j_1}(y_1)\stackrel{\eqref{eq:Tx}}{=} f_{j_1}(Tx)\stackrel{\eqref{eq:Tx2}}{=} f_{j_1}\circ f_{j_2}\circ f_{j_3} \ldots \circ f_{j_k} (y_2)$$
	\end{proof}

	\subsection{Construction of $\varphi_1$}\label{sub:phi}
	We now show how to construct a CPwL indicator function $\varphi $ for the set $C_1$, using the same methodology as in \cite{shahar}.
	By assumption, $C_0$ is a union of some $p_0 $ convex polytopes, each defined as an intersection of at most $m $ half-spaces.
	Accordingly $C_1$ is a finite union of $p=Jp_0 $ such polytopes.

	A half-space $H$ is defined by a linear function $\varphi_{H}:\RR^d\to \RR $
	\[\varphi_{H} \defeq \langle a, x \rangle + b,\]
	and $$H=\{x| \, \varphi_{H}(x) \geq 0 \} .$$
	Thus, by definition $\varphi_H$ is a CPwL indicator of $H$ and is a member of $\NNF{1,0}{d}{1} $. If $\P$ is a convex polytope which is an intersection of at most $m$ half-spaces $H_i $, then we obtain a CPwL indicator for $\P$ via
	$$\varphi^{}_{\P}(x)=\min(\varphi_{H_1}(x),\varphi_{H_2}(x),\ldots \varphi_{H_m}(x)). $$
	Note that $\varphi_{\P}\in \NNF{m,\lceil \log_2 m \rceil}{d}{1}$ due to the splitting, composition and minimum rules from Section~\ref{sec:prelim}.

	Now, if a set $C$ is the union of $p$ polytopes $\P_i $, each defined by at most $m$ half-spaces, then a CPwL indicator $\varphi_C$ for $C$ is given by
	$$\varphi_C(x)=\max(\varphi_{\P_1}(x),\varphi_{\P_2}(x),\ldots, \varphi_{\P_p}(x) ) ,$$
	which is in $\NNF{mp, \lceil \log_2 m \rceil+ \lceil \log_2 p \rceil}{d}{1}$ due to the splitting, composition and maximum rules from Section~\ref{sec:prelim}.
	This is the class of function $\varphi_1=\varphi_{C_1} $ belongs to, when we set $p=Jp_0 $.

	\subsection{Construction of $ T $ as NNF}\label{sub:T}
	We now construct $ T $ as an NNF. $T$ will be of the form
	$
	T = \sum_{j=1}^J T_j
	$
	where each $T_i$ will be an NNF satisfying
	\begin{equation}\label{eq:Ti}T_i(x) =\delta_{ij}f_i^{-1}(x), \quad  \forall x \in \bigcup_{j=1}^Jf_j C_0
	\end{equation}
	Let,
	$$M=\|f_i^{-1}\|_{L^{\infty}(C_0)}, \text{ and } -\delta=\max_{x \in \cup_{j \neq i}f_jC_0} \varphi^{}_{f_iC_0}(x)<0, $$
	and $\one_d$ denotes the constant vector with unit value in all $d$ entries . Recall that $ f_i^{-1} $ is defined for all points in $ \RR $, as it is the inverse of an affine transformation.
	Then, we construct $T_i$ as NNF according to the following scheme
	\begin{equation}
	\begin{array}{c}
	\begin{array}{*2c | c | c | c | c | c | *2c}
	\cline{3-3}\cline{5-5}\cline{7-7}
	&&\stackrel{}{f_i^{-1}(x)}&\rightarrow&f_i^{-1}(x)&&&&\\
	&\nearrow&&&&\searrow&&&\\
	x&\rightarrow&\underbrace{\varphi_{f_i C_0}(x)}&\rightarrow&\underbrace{\frac{4M}{\delta}(y_1 + \delta/2)}&\rightarrow&\underbrace{m\{f_i^{-1}(x), y_2\}}&\rightarrow&\cdots\\
	&&y_1\in\RR&&y_2\in\RR&&y_3\in\RR^d&& \\ \cline{3-3}\cline{5-5}\cline{7-7}
	\end{array}\\ \\
	\begin{array}{cc|c|ccccc}\cline{3-3}
	\cdots&\rightarrow&\stackrel{}{\relu(2(y_3+M\one_d))}-\relu(y_3+2M\one_d)&&&&&\\
	\cline{3-3}
	\end{array}
	\end{array}
	\end{equation}
	For convenience, we enumerate the sequential operations (marked by the boxes) by $ \ell_1,\ldots,\ell_4 $.
	We claim that $T_i$ constructed in this form  fulfills \eqref{eq:Ti}.
	Note that the magnitudes of the coordinates of $f_i^{-1}(x) $ are smaller than $M$ for all $x \in C_0 $.
	By construction, $y_2 > 2M $ if $x\in f_i(C_0) $ and is smaller than $-2M $ if $x \in f_j(C_0)$ for $j \neq i $.
	It then follows that if $x \in f_i(C_0)$ then $y_3=f_i^{-1}(x) $ and if $x \in f_j(C_0)$ for $j \neq i $ then $y_3\leq -2M $ elementwise. Finally by applying the function
	\begin{equation}
	\relu(2(t+M))-\relu(t+2M)=
	\left\{
	\begin{array}{lll}
	0      & \mbox{if } t \leq -2M \\
	-t-2M   & \mbox{if } -2M \leq t \leq -M\\
	t      & \mbox{if } t \geq -M
	\end{array}
	\right.
	\end{equation}
	elementwise to $y_3$, we obtain $T_i(x) $ which satisfies  \eqref{eq:Ti}. Using the parameter counting rules from Section~\ref{sec:prelim} we have
	\begin{equation*}
	\ell_1 = x^{f_i^{-1}}_{\varphi_{f_i C_0}} \in \NNF{d+mp_0, \,\lceil \log_2 m \rceil+ \lceil \log_2 p_0 \rceil}{d}{d+1}
	\end{equation*}
	\begin{equation*}
	\ell_2 \in \NNF{d+1,0}{d+1}{d+1}; \quad \ell_3 \in \NNF{2d,1}{d+1}{d}; \quad \ell_4 \in \NNF{2d,1}{d}{d}
	\end{equation*}
	and so using the composition we have that  $T_i\in \NNF{1/JW_T,L_T}{d}{d}$ for
	\begin{equation}
	W_T=J\max\{d+mp_0,2d \}, \, L_T=2+ \lceil \log_2 m \rceil+ \lceil \log_2 p_0 \rceil,
	\end{equation}
	and by using the summation rule to count the parameters of $T$ we obtain that $T \in \NNF{W_T,L_T}{d}{d}$.

	\subsection{Construction of $ C_k $ through NNF}\label{sub:phikConstruction}
	To conclude the proof of the theorem, we show how to construct a CPwL indicator of $C_k$ as an NNF based on \eqref{eq:phik} and our construction of $T$ and $\varphi$. We set
	$$m_0(x)=\varphi_1(x), \quad m_j(x)=m_j(m_{j-1},T^j x)=\min\{m_{j-1}(x),\varphi_1 (T^{j}x)\}, \quad  1 \leq j \leq k-1 .$$
	and note that $m_j\in \NNF{mp, \, \lceil \log_2 m \rceil+\lceil \log_2 p \rceil+1}{d+1}{1} $ while $m_0$ has one layer less. Note that the indicator $\varphi_k$ from \eqref{eq:phik} is equal to $m_{k-1}$, and so $\varphi_k$ can be computed by the following NNF:

	\begin{equation}\label{eq:rnn}
	\begin{array}{llllllllllll}
	& &Tx &\rightarrow & T^2x& \rightarrow& \cdots& & T^{k-1}x& &\\
	&\nearrow & &\searrow & & \searrow& & & &\searrow &\\
	x&\rightarrow &m_0(x) & \rightarrow & m_1(x) & \rightarrow & \cdots & & m_{k-2}(x) &\rightarrow &m_{k-1}(x)\\

	\end{array}
	.\end{equation}
	The width of $m_j$ is the same as the width of $\varphi_1$ and it is one layer deeper, therefore $\varphi_k(x)=m_{k-1}(x) $  can be constructed as an NNF with constant width $W_0$ and depth $kL_0 $ , where these constants are defined by
	\begin{equation}\label{eq:constants} W_0=J\max\{d+mp_0,2d\}+Jmp_0, \, L_0=\lceil\log_2 m \rceil+\lceil\log_2(Jp_0) \rceil+2 .\end{equation}

	%



	\section{Future work} Our results can be interpreted as showing that black and white images of 2D fractals- the indicator functions of $C_k$- can be compressed efficiently by DNN.
	DNN are also known to provide efficient construction of wavelets \cite{DBLP:journals/corr/abs-1901-02220}. As both IFS and wavelets have been used successfully to compress images we believe it may be possible to devise improved compression algorithms based on DNNs. We intend to investigate this in the near future.
	\bibliographystyle{apalike}
	\bibliography{fractalbib}

	\appendix
	\section{Additional proofs}\label{sec:additional}
	\paragraph{Proof of Theorem~\ref{thm:open}.}
	Looking at the outline of the proof of Theorem~\ref{thm:strong}, the only item that needs to be adapted to the case of open sets is the construction of $ T $, as it is not clear how we should define it on the boundary of the domains $ f_j U_0 $.
	Assume w.l.o.g. that $0 \not \in \bar U_1 $. We define $T$ to be a (non-continuous) function which is equal to $f_j^{-1}(x) $ if $x \in f_j U_0 $ for some $j$, and is equal to zero otherwise. Note that Lemma~\ref{lem:CkiffTk} does not assume compactness of $C_0$, and so the sufficient and necessary conditions of item \ref{outline:TCkEquivalence} of the outline holds for the non-continuous $T$.

	Let $\varphi$ be the CPwL indicator function of $\bar U_1 $, which is constructed in Subsection~\ref{sub:phi}, and note that $\varphi$ satifies that $x \in U_1 $ if and only if $\varphi(x)>0$. Thus we can use \eqref{eq:phik} to define a (non-continuous) function $\varphi_k$ which satisfies $\varphi_k(x)>0 $ if and only if $x \in U_k $.

	Next we can use the same construction as in Subsection~\ref{sub:T} to define for $\delta>0 $ CPwL maps $T_i^{\delta}(x) $ which are zero if $\varphi_i(x)\leq 0 $, and are equal to $f_i^{-1} $ if $\varphi_i(x) \geq \delta $. We take the sum of these maps to obtain a CPwL map $T^\delta$ which converge pointwise to $T$ as $\delta \rightarrow 0 $. Next we replace $T$ with $T^\delta$ in the definition of $\varphi_k$ in \eqref{eq:phik} to obtain CPwL functions $\varphi^{\delta}_k$ which converge pointwise to $\varphi_k$ as $\delta \rightarrow 0$.

	To obtain pointwise convergence to the indicator of $U_k$ we need to compose $\varphi^{\delta}_k$ with a simple scalar function:
	For $a<b$ we define the function
	\begin{equation}\label{eq:fab}
	f_{a,b}(x)=\frac{1}{b-a}\left(\relu(x-a)-\relu(x-b) \right)=
	\left\{
	\begin{array}{lll}
	0      & \mbox{if } x \leq a \\
	\frac{x-a}{b-a}   & \mbox{if } a \leq x \leq b\\
	1      & \mbox{if } x \geq b
	\end{array}
	\right.
	\end{equation}
	Now we set $a= \varphi_k(0) $ which is negative since $0 \not \in  \bar U_1 $, and $b=0$, and obtain that $\psi_{\delta}=f_{a,b}\circ \varphi^{\delta}_k $ converges pointwise to the indicator of $U_k$ as $\delta \rightarrow 0 $. The $L^p(\K)$ convergence argument is identical to the argument in Lemma~\ref{lem:indicator}\qed

	\paragraph{Proof of Lemma~\ref{lem:indicator}.}  By choosing $f=f_{a,b} $ as defined in \eqref{eq:fab} with $a=-2, b=-1 $, we obtain
	that $\psi_t(x)=f \circ \varphi_C(tx) $ converges elementwise to $\chi_C$ as $t \rightarrow \infty$. For a fixed compact set $\K$ and fixed $1 \leq p < \infty $, we show $L^p(\K) $ convergence by showing that the integral of the functions
	$$g_t(x)=|\psi_t(x)-\chi_C(x)|^p $$
	converge to zero. This follows easily from the dominated convergence theorem since  $g_t $ converge pointwise to zero and are bounded uniformly for all $t$ by the constant function $1$ which is integrable in $\K$. \qed

	\paragraph{Dimension argument}
	\begin{Lemma}\label{lem:dim}
		If $\ell$ is larger than the number of parameters $p$ defining $\NNF{W,L}{1}{1} $, then not all \emph{CPwL} functions with $\ell$ knots are contained in $\NNF{W,L}{1}{1} $.
	\end{Lemma}
	\begin{proof}
		Denote the set of parameters by $\theta$, and denote by $f_\theta$ the NNF defined through $\theta$.
		Choose any $t_1,t_2,\ldots,t_\ell \in \RR $.
		It is sufficient to show that the function $F:\RR^p \to \RR^\ell$ defined by
		$$F_i(\theta)= f_{\theta}(t_i), i=1,\ldots \ell $$
		is not onto.
		Indeed, $F$ is continuous as the parameter space $\RR^p$ can be partitioned into a finite number of sets on which $F$ is a multivariate polynomial whose degree is determined by the degree of the network (for example, the composition of two affine transformations results with a second order polynomial in the coefficients).
		It follows that the restriction of $F$ to a closed ball $\bar B_N$ of radius $N$ is Lipschitz, and therefore it is known (see Chapter 2 in  \cite{falconer2004fractal}) that its image has Hausdorff dimension $\leq p $.
		The image of $F$ is the countable  union of $F(B_N) $, hence it has Hausdorff dimension $\leq p < \ell $ as well  and so $F$ is not onto.
	\end{proof}

	\section{IFS examples}
	\textbf{The Sierpinski triangle:}
	The Sierpinski triangle $K_{\textrm{Sierpinski}}$ is the attractor of an IFS on $\RR^2$.
	It is defined using an equilateral triangle $T$ with vertices $v_1,v_2, v_3 \in \RR^2 $, where $ \IFS{F}{f_1, f_2, f_3} $ are defined to be
	$$f_i(x)=1/2x+v_i,~~ i=1,2,3.$$
	The open triangle $U_0=T^\circ$ is a nice open set.
	Figure~\ref{fig:IFS}(b) shows the first few sets  $U_i$ obtained from this choice of $U_0$.

	\textbf{The Koch curve:}
	The Koch curve $K_{\textrm{Koch}}$ is the attractor of an IFS on $\RR^2$.
	It is defined via four maps $f_i(x)=A_i(x)+b_i $ where
	\begin{align*}
	A_1&=\frac{1}{3}
	\begin{bmatrix}
	1 & 0 \\
	0 & 1
	\end{bmatrix}
	, \, b_1=\begin{bmatrix}
	0 \\
	0
	\end{bmatrix}
	, \quad A_2=
	\begin{bmatrix}
	1/6 & -\sqrt{3}/6 \\
	\sqrt{3}/{6}  & {1}/{6}
	\end{bmatrix}
	, \,
	b_2=\begin{bmatrix}
	{1}/{3} \\
	0
	\end{bmatrix}\\
	A_3&=A_2^T,
	\, b_3=\begin{bmatrix}
	1/2 \\
	\sqrt{3}/6
	\end{bmatrix}
	, \quad A_4=A_1, \,
	b_4=\begin{bmatrix}
	2/3 \\
	0
	\end{bmatrix}.
	\end{align*}
	The open triangle with vertices
	$$v_1=\begin{bmatrix}
	0\\
	0
	\end{bmatrix}
	, \, v_2=\begin{bmatrix}
	1/2\\
	\sqrt{3}/2
	\end{bmatrix}
	, \,v_3=\begin{bmatrix}
	1\\
	0
	\end{bmatrix}
	$$
	is a nice open set w.r.t. this IFS.
	Figure~\ref{fig:IFS}(c) shows the first few sets  $U_i$ obtained from this choice of $U_0$.

	\section{Lifting}\label{app:lift} Any IFS on $ \RR^d $ composed of affine and invertible functions can be lifted to an IFS on $ \RR^{d+1} $ fulfilling the conditions of \ref{thm:strong}.
	Assume we are given such an IFS $\{f_1,\ldots,f_J \}$ on $\RR^d$ and an initial convex polytope $C_0$ fulfilling \eqref{eq:contained} (such a set always exists). We can then define a new IFS $\{\hat f_1,\ldots, \hat f_J \}$ on $\RR^{d+1}$ by
	$$\hat f_j(x,t)=\left( f_j(x), \, \frac{1}{3(J+1)}t+\frac{j}{J+1} \right) .$$
	The convex polytope $\hat C_0=C_0 \times [0,1] $ is a nice compact subset of $\RR^{d+1}$, and so we can use Theorem~\ref{thm:strong} to construct the indicator function of  $\hat C_k$, which is related to the indicator function of $C_k$ via
	$$\chi_{C_k}(x)=\max_{t \in [0,1]} \chi_{\hat C_k}(x,t) $$
\end{document}